%% file: master.tex
\documentclass{article}
\usepackage{authblk}
\usepackage[utf8]{inputenc} 
\usepackage[T1]{fontenc}    
\usepackage{hyperref}       
\usepackage{url}            
\usepackage{booktabs}       
\usepackage{amsfonts}       
\usepackage{nicefrac}       
\usepackage{microtype}      

\usepackage{fullpage}
\usepackage{authblk}
\usepackage{graphicx}										
\usepackage{subfigure}										
\usepackage{lmodern}		
\usepackage{enumitem}		
\usepackage{pgfplots}		
\pgfplotsset{compat=1.8}	
	\usetikzlibrary{calc, shapes.arrows}
\usepackage{ifthen}			
\usepackage[noend]{algpseudocode}				
\algnewcommand{\param}[1]{						
	\textbf{parameter: }#1						
}												
\algnewcommand{\params}[1]{					
	\textbf{parameters: }#1						
}												
\algnewcommand{\init}[1]{						
	\textbf{initialize: }#1						
}												
\usepackage{float}								
\floatstyle{ruled}								
\newfloat{algorithm}{tbp}{loa}					
\providecommand{\algorithmname}{Algorithm}	
\floatname{algorithm}{\protect\algorithmname}	
\usepackage[numbers]{natbib}				
\usepackage{amsmath,amssymb,amsthm}
\usepackage{color}	
\usepackage{dsfont}
\usepackage{todonotes}

\newlength{\minipagewidth} 
\DeclareMathOperator*{\argmin}{arg\,min}

\newcommand{\bx}{\boldsymbol{x}}
\newcommand{\bu}{\boldsymbol{u}}
\newcommand{\bw}{\boldsymbol{w}}

\newcommand{\bzero}{\boldsymbol{0}}
\newcommand{\bth}{\boldsymbol{\theta}}

\newcommand{\field}[1]{\mathbb{#1}}
\newcommand{\R}{\field{R}}

\newcommand{\E}{\field{E}}

\newcommand{\I}{\field{I}}
\renewcommand{\Pr}{\field{P}}
\renewcommand{\P}{\field{P}}

\newcommand{\cf}{\mathcal{F}}

\newcommand{\scO}{\mathcal{O}}

\newcommand{\cn}{\mathcal{N}}

\newcommand{\cx}{\mathcal{X}}

\newcommand{\norm}[1]{\left\|{#1}\right\|}

\newcommand{\dt}{\displaystyle}

\newcommand{\s}{\subset}
\newcommand{\m}{\setminus}

\newcommand{\e}{\varepsilon}
\newcommand{\bool}{\{0,1\}}

\renewcommand{\l}{\ldots}
\newcommand{\iop}{\infty}

\newcommand{\dif}{\mathrm{d}}
\newcommand{\bad}{J_\mathrm{B}}

\newcommand{\chibar}{\overline{\chi}}

\newtheorem{lemma}{Lemma}
\newtheorem{theorem}{Theorem}

\theoremstyle{definition}

\newtheorem{assumption}{Assumption}

\title{Cooperative Online Learning: \\ Keeping your Neighbors Updated}

\author[1]{Nicol\`{o} Cesa-Bianchi}
\author[2]{Tommaso R.~Cesari}
\author[3]{Claire Monteleoni}

\affil[1]{\small Dipartimento di Informatica \& DSRC, Università degli Studi di Milano, Milano, Italy}
\affil[2]{\small Dipartimento di Informatica \& DSRC, Università degli Studi di Milano, Milano, Italy\authorcr
				 ANITI, Artificial and Natural Intelligence Toulouse Institute, Toulouse, France\authorcr
				 TSE, Toulouse School of Economics, Toulouse, France}
\affil[3]{\small Department of Computer Science, University of Colorado Boulder, Colorado}

\begin{document}

\maketitle

\input{abstract}
\input{intro}
\input{related}
\input{prelim}
\input{omd}
\input{stoc}
\input{lower}
\input{adv}
\input{conclusions}

\section*{Acknowledgments}
Nicolò Cesa-Bianchi and Tommaso Cesari gratefully acknowledge partial support by the GAMENET COST Action (CA 16228) and by the MIUR PRIN grant \textsl{Algorithms, Games, and Digital Markets} (ALGADIMAR). Nicolò Cesa-Bianchi is also partially supported by the Google Focused Award \textsl{Algorithms and Learning for AI} (ALL4AI).

\bibliography{tc,ncb,cm}

\appendix

\input{manyInApp}

\end{document}

%% file: abstract.tex

\begin{abstract}
\noindent We study an asynchronous online learning setting with a network of agents. At each time step, some of the agents are activated, requested to make a prediction, and pay the corresponding loss. The loss function is then revealed to these agents and also to their neighbors in the network. Our results characterize how much knowing the network structure affects the regret as a function of the model of agent activations. When activations are stochastic, the optimal regret (up to constant factors) is shown to be of order $\sqrt{\alpha T}$, where $T$ is the horizon and $\alpha$ is the independence number of the network. We prove that the upper bound is achieved even when agents have no information about the network structure. When activations are adversarial the situation changes dramatically: if agents ignore the network structure, a $\Omega(T)$ lower bound on the regret can be proven, showing that learning is impossible. However, when agents can choose to ignore some of their neighbors based on the knowledge of the network structure, we prove a $\scO(\sqrt{\chibar T})$ sublinear regret bound, where $\chibar \ge \alpha$ is the clique-covering number of the network.
\end{abstract}

%% file: intro.tex

\section{Introduction}
\label{s:intro}
Distributed asynchronous online learning settings with communication constraints arise naturally in several applications. For example, large-scale learning systems are often geographically distributed, and in domains such as finance or online advertising, each agent must serve high volumes of prediction requests. If agents keep updating their local models in an online fashion, then bandwidth and computational constraints may preclude a central processor from having access to all the observations from all sessions, and synchronizing all local models at the same time. An example in a different domain is mobile sensor networks cooperating towards a common goal, such as environmental monitoring. Sensor readings provide instantaneous, full-information feedback and energy-saving constraints favor short-range communication over long-range. 
Online learning algorithms distributed over spatial locations have already been proposed for problems in the field of climate informatics by \citet{mm12, mm17}, and have shown empirical performance advantages compared to their global (i.e., non-spatially distributed) online learning counterparts.

Motivated by these real-life applications, we introduce and analyze an online learning setting in which a network of agents solves a common online convex optimization problem by sharing feedback with their network neighbors.
Agents do not have to be synchronized. At each time step, only some of the agents are requested to make a prediction and pay the corresponding loss: we call these agents ``active''. Because the feedback (i.e., the current loss function) received by the active agents is communicated to their neighbors, both active agents and their neighbors can use the feedback to update their local models. The lack of global synchronization implies that agents who are not requested to make a prediction get ``free feedback'' whenever someone is active in their neighborhood. Since in online convex optimization the sequence of loss functions is fully arbitrary, it is not clear whether this free feedback can improve the system's performance. In this paper, we characterize under which conditions and to what extent such improvements are possible.

Our goal is to control the network regret, which we define by summing the average instantaneous regret of the active agents at each time step. In order to build some intuition on this problem, consider the following two extreme cases where, for the sake of simplicity, we assume exactly one agent is active at each time step. If no communication is possible among the agents, then each agent $v$ learns in isolation over the subset $T_v$ of time steps when they are active. Assuming each agent runs a standard online learning algorithm with regret bounded by $\scO(\sqrt{T})$ ---such as Online Mirror Descent (OMD)--- the network regret is at most of order $\sum_v \sqrt{T_v} \le \sqrt{NT}$, where $T = \sum_v T_v$ and $N$ is the number of agents. Next, consider a fully connected graph, where agents share their feedback with the rest of the network. Each local instance of OMD now sees the same loss sequence as the other instances, so the sequence of predictions is the same, no matter which agents are chosen to be active. The network regret is then bounded by $\scO(\sqrt{T})$, as in the single-instance case. Our goal is to understand the regret when the communication network corresponds to an arbitrary graph $G$.

We consider two natural activation mechanisms for the agents: stochastic and adversarial. In the stochastic setting, at each time step $t$ each agent $v$ is independently active with probability $q_v$, where $q_v$ is a fixed and unknown number in $[0,1]$. Under this assumption, we show that when each agent runs OMD, the network regret is $\scO(\sqrt{\alpha T})$, where $\alpha \le N$ is the independence number of the communication graph. Note that this bound smoothly interpolates the two extreme cases of no communication ($\alpha = N$) and full communication ($\alpha = 1$). From this viewpoint, $\alpha$ can be viewed as the number of ``effective instances'' that are implicitly maintained by the system. It is not hard to prove that this upper bound cannot be improved upon: fix a network $G$ and a maximal independent set in $G$ of size $\alpha$. Define $q_v = 1/\alpha$ if $v$ belongs to the independent set and $0$ otherwise. Then no two nodes that can ever become active are adjacent in $G$, and we reduced the problem to that of learning with $\alpha$ non-commmunicating agents over $T/\alpha$ time steps. Since there are instances of the standard online convex optimization problem on which any agent strategy has regret $\Omega(\sqrt{T})$, we obtain that the network regret must be at least of order $\alpha\sqrt{T/\alpha} = \sqrt{\alpha T}$. 
Note that this lower bound also applies to algorithms that have complete preliminary knowledge of the graph structure, and can choose to ignore or process any feedback coming from their neighbors. In contrast, the OMD instances used to prove the upper bound are fully oblivious both to the graph structure and to the source of their feedback (i.e., whether their agent is active as opposed to being the neighbor of an active agent).

In the adversarial activation setting, nodes are activated according to some unknown deterministic schedule. Surprisingly, under the same assumption of obliviousness about the feedback source which we used to prove the $\scO(\sqrt{\alpha T})$ upper bound for stochastic activations, we show that on certain network topologies a deterministic schedule of activations can force a \textsl{linear} regret on \textsl{any} algorithm, thus making learning impossible. On the other hand, if agents are free to use feedback only from a subset of their neighbors chosen with knowledge of the graph structure, then the network regret of OMD is $\scO(\sqrt{\chibar T})$, where $\chibar$ is the clique-covering number of the communication graph. Hence, unlike the stochastic case, where the knowledge of the graph is not required to achieve optimality, in the adversarial case the ability of choosing the feedback source based on the graph structure is both a necessary and sufficient condition for sublinear regret.

The extension of the OMD analysis to a multiagent setting with communication (Theorem~\ref{t:upper-coop}), and the lower bound for the adversarial activation setting (Theorem~\ref{th:lower-adv}) are the main technical novelties of the paper.

%% file: related.tex

\section{Related Works}
\label{s:related}
The study of cooperative nonstochastic online learning on networks was pioneered by~\citet{awerbuch2008competitive}, who investigated a bandit setting in which the communication graph is a clique, users are clustered so that the loss function at time $t$ may differ across clusters, and some users may be non-cooperative.
More recently, a similar line of work was pursued by \citet{cesa2016delay}, where they derive graph-dependent regret bounds for nonstochastic bandits on arbitrary networks when the loss function is the same for all nodes and the feedbacks are broadcast to the network with a delay corresponding to the shortest path distance on the graph. Although their regret bounds ---like ours--- are expressed in terms of the network independence number, this happens for reasons that are very different from ours, and by means of a different analysis. In their setting all agents are simultaneously active at each time step, and sharing the feedback serves the purpose of reducing the variance of the importance-weighted loss estimates. A node with many neighbors observes the current loss function evaluated at all the points corresponding to actions played by the neighbors. Hence, in that context cooperation serves to bring the bandit feedback closer to a full information setting. 

In contrast, we study a full information setting in which agents get free and meaningful feedback only when they are not requested to predict.\footnote{%
Two adjacent agents that are simultaneously active exchange their feedback, but this does not bring any new information to either agent because we are in a full information setting and the loss function is the same for all nodes.
}
Therefore, in our setting cooperation corresponds to faster learning (through the free feedback that is provided over time) \textsl{within the 
full information model}, as opposed to \citep{cesa2016delay} where cooperation \textsl{increases feedback within a single time-step}. An even more recent work considering bandit networks studies a stochastic bandit model with simultaneous activation and constraints on the amount of communication between neighbors \citep{martinez2018decentralized}. Their regret bounds scale with the spectral gap of the communication network. The work of \citet{sahu2017dist} investigates a different partial information model of prediction with expert advice where each agent is paired with an expert, and agents see only the loss of their own expert. The communication model includes delays, and the regret bound depends on a quantity related to the mixing time of a certain random walk on the network. \citet{zhao2019decentralized} study a decentralized online learning setting in which losses are characterized by two components, one adversarial and another stochastic. They show upper bounds on the regret in terms of a constant representing the magnitude of the adversarial component and another constant measuring the randomness of the stochastic part.

The idea of varying the amount of feedback available to learning agents has also appeared in single-agent settings. In the \textsl{sleeping experts} model \citep{freund1997using}, different subsets of actions are available to the learner at different time steps. In our multi-agent setting, instead, actions are always available while the agents are occasionally sleeping. An algorithmic reduction between the two settings seems unlikely to exist because actions and agents play completely different roles in the learning process. In the \textsl{learning with feedback graphs} model \citep{alon2015online,mannor2011bandits}, each selection of an action reveals to the learner the loss of the actions that are adjacent to it in a given graph. In our model, each time an active agent plays an action, the loss vector is revealed to the agents that are adjacent to the active learner. There is again a similarity between actions and agents in the two settings, but to the best of our knowledge there is no algorithmic reduction from multi-agent problems to single-agent problems. Yet, it should not come as a surprise that some general graph-theoretic tools ---like Lemma~\ref{lm:bound-alpha}--- are used in the analysis of both single-agent and multi-agent models.

A very active area of research involves distributed extensions of online convex optimization,  in which the global loss function is defined as a sum of local convex functions, each associated with an agent. Agents are run over the local optimization problem corresponding to their local functions and communicate with their neighborhood to find a point in the decision set approximating the loss of the best global action. This problem has been studied in various settings: distributed convex optimization ---see, e.g., \citep{duchi2012dual,scaman2018optimal} and references therein, distributed online convex optimization \citep{hosseini2013online}, and a dynamic regret extension of distributed online convex optimization \citep{shahrampour2018distributed}. Unlike our work, these papers consider distributed extensions of OMD (and Nesterov dual averaging) based on generalizations of the consensus problems. The resulting performance bounds scale inversely in the spectral gap of the communication network.

%% file: prelim.tex

\section{Preliminaries and definitions}
\label{s:prelim}
Let $G = (V,E)$ be a communication network, i.e., an undirected graph over a set $V$ of $N$ \textsl{agents}. Without loss of generality, assume $V= \{1,\l,N\}$. For any agent $v \in V$, we denote by $\cn_v$ the set of nodes containing the agent $v$ and the neighborhood $\big\{w\in V \mid (v,w) \in E\big\}$. The \textsl{independence number} $\alpha_G$ is the cardinality of the biggest \textsl{independent set} of $G$, i.e., the cardinality of the biggest subset of agents, no two of which are neighbors.

We study the following cooperative online convex optimization protocol: initially, hidden from the agents, the environment picks a sequence of subsets $S_1,S_2, \l \subseteq V$ of \textsl{active} agents and a sequence of differentiable convex real loss functions $\ell_1, \ell_2, \l$ defined on a convex decision set $\cx \subset \R^d$. Then, for each time step $t\in\{1,2,\l\}$,
\begin{enumerate}[topsep=0pt,parsep=0pt,itemsep=0pt]
\item each agent $v\in S_t$ predicts with $\bx_t (v) \in \cx$,
\item each agent $v \in \bigcup_{v\in S_t} \cn_v$ receives $\ell_t$ as feedback,
\item the system incurs the loss $\frac{1}{|S_t|}\sum_{v\in S_t} \ell_t\big(\bx_t(v)\big)$ (defined as $0$ when $S_t\equiv\varnothing$).
\end{enumerate}
We assume each agent $v$ runs an instance of the same online algorithm. Each instance learns a local model generating predictions $\bx_t(v)$. This local model is updated whenever a feedback $\ell_t$ is received. We call \textsl{paid feedback} the feedback $\ell_t$ received by $v$ when $v \in S_t$ (i.e., the agent is active) and \textsl{free feedback} the feedback $\ell_t$ received by $v$ when $v \in \big(\bigcup_{v\in S_t} \cn_v\big)\setminus\{S_t\}$ (i.e., the agent is not active but in the neighborhood of some active agent). The goal is to minimize the \textsl{network regret} as a function of the unknown number $T$ of time steps,
\begin{equation}
\label{eq:regr}
	R_T
= 
	\sum_{t=1}^T \frac{1}{|S_t|}\sum_{v\in S_t} \ell_t\big(\bx_t(v)\big) - \inf_{\bx \in \cx} \sum_{t=1}^T \ell_t (\bx)
\end{equation}
Note that only the losses of active agents contribute to the network regret.

%% file: omd.tex

\newcommand{\bp}{\boldsymbol{p}}

\section{Online Mirror Descent}
\label{s:omd}
%

\begin{algorithm}
	\textbf{Parameters:} $\sigma_t$-strongly convex regularizers $g_t\colon\cx \to \R$ for $t\in\{1,2,\l\}$\par
	\textbf{Initialization:} $\bth_1 = \bzero \in \R^d$\par
	\begin{algorithmic}[1]
	\For{$t\in\{1,2,\l\}$}
		\State choose $\bw_t = \nabla g_t^*(\bth_t)$
		\State observe $\nabla \ell_t(\bw_t) \in \R^d$
		\State update $\bth_{t+1} = \bth_t - \nabla \ell_t(\bw_t)$
	\EndFor
	\end{algorithmic}
\caption{Online Mirror Descent\label{a:omd}}
\end{algorithm}
We now review the standard Online Mirror Descent algorithm (OMD) ---see Algorithm~\ref{a:omd}--- and its analysis. Let $f\colon \cx \to \R$ be a convex function. We say that $f^*\colon \R^d \to \R$ is the \textsl{convex conjugate} of $f$ if
$
f^*(\bx) = \sup_{\bw \in \cx} \big( \bx \cdot \bw - f(\bw) \big)
$.
We say that $f$ is $\sigma$\textsl{-strongly convex} on $\cx$ with respect to a norm $\norm{\cdot}$ if there exists $\sigma \ge 0$ such that, for all $\bu,\bw\in \cx$ we have
$
	f(\bu) \ge f(\bw) + \nabla f (\bw) \cdot (\bu - \bw) + \frac{\sigma}{2} \norm{ \bu - \bw }^2
$.
The following well-known result can be found in \citep[Lemma~2.19 and subsequent paragraph]{shalev2012online}.
\begin{lemma}
Let $f\colon \cx \to \R$ be a strongly convex function on $\cx$. Then the convex conjugate $f^*$ is everywhere differentiable on $\R^d$.
\end{lemma}
The following result ---see, e.g., \citep[bound~(6) in Corollary~1 with $F$ set to zero]{orabona2015generalized}--- shows an upper bound on the regret of OMD.
\begin{theorem}
\label{t:omd}
Let $g\colon \cx\to\R$ be a differentiable function $\sigma$-strongly convex with respect to $\norm{\cdot}$. Then the regret of OMD run with  $g_t = \frac{\sqrt{t}}{\eta}g$, for $\eta > 0$, satisfies
\begin{align*}
	\sum_{t=1}^T \ell_t\big(\bx_t\big) &- \inf_{\bx \in \cx} \sum_{t=1}^T \ell_t (\bx)
\le
	\frac{D}{\eta}\sqrt{T}
	+\frac{\eta}{2\sigma}\sum_{t=1}^{T}\frac{1}{\sqrt{t}}\norm{\nabla\ell{}_{t}}_{*}^{2}
\end{align*}
where $D = \sup g$ and $\norm{\cdot}_{*}$ is the dual norm of $\norm{\cdot}$. If $\sup \norm{\nabla\ell{}_{t}}_{*} \le L$, then choosing $\eta = \sqrt{2 \sigma D}/L$ gives
$
	R_{T} \le L\sqrt{2 D T / \sigma}
$.
\end{theorem}
A popular instance of OMD is the standard online gradient descent algorithm, corresponding to choosing $\cx$ equal to a closed Euclidean ball centered at the origin, and setting $g = \frac{1}{2}\norm{\cdot}^2$ for all $t$, where $\norm{\cdot}$ is the Euclidean norm. Another instance is the Hedge algorithm for prediction with expert advice, corresponding to choosing $\cx$ equal to the probability simplex, and setting $g(\bp) = \sum_i p_i\ln p_i$.

%% file: stoc.tex

\section{Stochastic Activations}
\label{s:stoc}
In this section we analyze the performance of OMD when the sets $S_t$ of active agents are chosen stochastically. As discussed in the introduction, in this setting we do not require any ad-hoc interface between each OMD instance and the rest of the network. In particular, we make the following assumption.
\begin{assumption}[Oblivious network interface]
An online algorithm $A$ is run with an \textsl{oblivious network interface} if for each agent $v$ it holds that:
\begin{enumerate}[topsep=0pt,parsep=0pt,itemsep=0pt]
\item $v$ runs an instance $A_v$ of $A$
\item $A_v$ uses the same initialization and learning rate as the other instances
\item $A_v$ makes predictions and updates while being oblivious to whether $v \in S_t$ or $v \in \underset{u\in S_t}{\bigcup} \cn_u\setminus S_t$
\end{enumerate}
\end{assumption}
This assumption implies that each instance is oblivious to both the network topology and the location of the agent in the network. Moreover, instances make an update whenever they have the opportunity to do so, (i.e., whenever they or some of their neighbors are active). The purpose of this assumption is to show that communication might help OMD even without any network-specific tuning. In concrete applications, one might use ad-hoc OMD variants that rely on the knowledge of the task at hand, and decrease the regret even further. However, the lower bound proven in Section~\ref{s:lower} shows that the regret cannot be decreased significantly even when agents have full knowledge of the graph.

We start by considering a slightly simplified stochastic activation setting, where only a single agent is activated at each time step (i.e., $|S_t|=1$ for all $t$). The more general stochastic case is analyzed at the end of this section.

We assume that the active agents $v_1,v_2,\l$ are drawn i.i.d.\ from an unknown fixed distribution $q$ on $V$.
%
The main result of this section is an upper bound on the regret of the network when all agents run the basic OMD (Algorithm~\ref{a:omd}) with an oblivious network interface. We show that in this case the network achieves the same regret guarantee as the single-agent OMD (Theorem~\ref{t:omd}) multiplied by the square root of independence number of the communication network.

Before proving the main result, we state a combinatorial lemma that allows to upper bound the sum of a ratio of probabilities over the vertices of an undirected graph with the independence number of the graph \citep{griggs1983lower,mannor2011bandits}. The proof is included for completeness.
\begin{lemma}
\label{lm:bound-alpha}
Let $G=(V,E)$ be an undirected graph with independence number $\alpha_G$ and $q$ any probability distribution on $V$ such that $Q_v = \sum_{w\in \cn_v}q_v > 0$ for all $v\in V$. Then
\[
	\sum_{v\in V} \frac{q_v}{Q_v}
\le
	\alpha_G
\]
\end{lemma}
\begin{proof}
Initialize $V_1=V$, fix $w_1\in\argmin_{w\in V_1}Q_w$, and denote $V_{2}=V\m \cn_{w_{1}}$. For $k\geq 2$ fix $w_{k}\in\argmin_{w\in V_{k}}Q_w$ and shrink $V_{k+1}=V_{k}\m \cn_{w_{k}}$ until $V_{k+1}=\varnothing$. Being $G$ undirected, we have $w_{k}\notin\bigcup_{s=1}^{k-1}\cn_{w_{s}}$, therefore the number $m$ of times that an action can be picked this way is upper bounded by $\alpha_G$. Denoting $\cn'_{w_k}=V_{k}\cap \cn_{w_k}$, this implies
\[
	\sum_{v\in V}\frac{q_v}{Q_v}
=
	\sum_{k=1}^{m}\sum_{v\in \cn'_{w_k}}\frac{q_v}{Q_v}
\le
	\sum_{k=1}^{m}\sum_{v\in \cn'_{w_k}}\frac{q_v}{Q_{w_{k}}}
\le
	\sum_{k=1}^{m}\frac{\sum_{v\in \cn_{w_k}}q_v}{Q_{w_{k}}}
=
	m
\le
	\alpha_G
\]
\end{proof}
The following holds for any differentiable function $g\colon \cx\to\R$, $\sigma$-strongly convex with respect to some norm $\norm{\cdot}$. 
\begin{theorem}
\label{t:upper-coop}
Consider a network $G = (V,E)$ of $N$ agents and assume $S_t = \{v_t\}$ for each $t$, where $v_t$ is drawn i.i.d.\ from some fixed and unknown distribution on $V$. If all agents run OMD with an oblivious network interface and using $g_t = \frac{\sqrt{t}}{\eta}g$, for $\eta > 0$, then the network regret satisfies
\[
	\E[R_T]
\le
	\left( \frac{D}{\eta}
	+\frac{\eta L^2}{2 \sigma} \right) \sqrt{\alpha_G T}
\]
where $D \ge \sup g$, $L \ge \sup \norm{\nabla\ell{}_{t}}_{*}$, and $\norm{\cdot}_{*}$ is the dual norm of $\norm{\cdot}$. In particular, choosing $\eta = \sqrt{2 \sigma D}/L$ gives
$
	\E[R_T] \le L\sqrt{2D \alpha_G T / \sigma}
$.
\end{theorem}
\begin{proof}
Fix $\bx \in \cx$, any sequence of realizations $v_1,\dots,v_T$, and any $v$ in the support $V'\s V$ of the activation distribution $q$. Note that the OMD instance run by $v$, makes an update at time $t$ only when $v \in \cn_{v_t}$. Hence, letting
$
	r_t(v)
=
	\ell_t\big(\bx_t(v)\big)
	-  \ell_t (\bx)
$
and applying Theorem~\ref{t:omd},
\begin{equation}
\label{e:base}
	\sum_{t=1}^T r_t(v) \I\{v \in \cn_{v_t} \}
\le
	\frac{D}{\eta}\sqrt{T_v}
	+\frac{\eta L^2}{2\sigma}\sum_{t=1}^{T}\frac{\I\{v \in \cn_{v_t} \}}{\sqrt{\sum_{s=1}^t\I\{v \in \cn_{v_s} \}}} 
\le
	\left(\frac{D}{\eta} + \frac{\eta L^2}{2\sigma}\right)\sqrt{T_v}
\end{equation}
where $T_v = \sum_{t=1}^{T} \I\{v \in \cn_{v_t} \}$, the addends after the first inequality are intended to be null when the denominator is zero, and we used $\sum_{s=1}^{T_v} s^{-\nicefrac{1}{2}} \le 2 \sqrt{T_v}$.
Note that 
$
	r_t(v)
$
is independent of $v_t$, as it only depends on the subset of $v_s$, $s\in\{1,\dots,t-1\}$, such that $v \in \cn_{v_s}$.
Denote by $Q_v$ the probability $\Pr(v \in \cn_{v_t}) = \sum_{w\in \cn_v}q(w)>0$. 
Let $\cf_{t-1}$ be the $\sigma$-algebra generated by $\{v_1, \l, v_{t-1}\}$. Since $Q_v$ is independent of $t$, $\Pr\big( v \in \cn_{v_t} \mid \cf_{t-1} \big) = Q_v$. Therefore, taking expectation with respect to $v_1,\dots,v_T$ on both sides of \eqref{e:base}, using $\E[T_v] = Q_v T$, and applying Jensen's inequality, yields
\begin{align}
\label{e:fdm}
	\E &\left[ \sum_{t=1}^{T} r_t(v) Q_v \right]
\le
	\left(\frac{D}{\eta} + \frac{\eta L^2}{2 \sigma}\right) \sqrt{Q_v T}
\end{align}
Now, letting $R_T(\bx) = \sum_{t=1}^T r_t(v_t)$, we have that $\E\big[R_T(\bx)\big]$ is equal to
\begin{align*}
	\E \left[ \sum_{v\in V'} \sum_{t=1}^T r_t (v) \I \{v_t = v\} \right]
= 
	\E \left[ \sum_{v\in V'} \sum_{t=1}^T r_t (v) \E\big[ \I \{v_t = v\} \mid \cf_{t-1} \big] \right]
= 
	\sum_{v\in V'} q_v \E \left[ \sum_{t=1}^T r_t (v) \right]
\end{align*}
Dividing both sides of \eqref{e:fdm} by $Q_v > 0$, we can write
\begin{align*}
	\E\big[R_T(\bx)\big]
\le
	\left(\frac{D}{\eta} + \frac{\eta L^2}{2 \sigma}\right) \sum_{v\in V'} q_v \sqrt{\frac{T}{Q_v}}
\le
	\left(\frac{D}{\eta} + \frac{\eta L^2}{2 \sigma}\right) \sqrt{T\sum_{v\in V'}\frac{q_v}{Q_v}}
\le
	\left(\frac{D}{\eta} + \frac{\eta L^2}{2 \sigma}\right) \sqrt{\alpha T}
\end{align*}
where in the last two inequalities we applied Jensen's inequality and Lemma~\ref{lm:bound-alpha}. Observing that $\E[R_T] = \sup_{\bx\in\cx} \E\big[R_T(\bx)\big]$ and recalling that $\bx$ was chosen arbitrarily in $\cx$ concludes the proof.
\end{proof}
Note that the proof of the previous result gives a tighter upper bound on the network regret in terms of the independence number $\alpha' \le \alpha$ of the subgraph induced by the support $V'$ of $q$.

Next, we consider the setting in which we allow the activation of more than one agent per time step. At the beginning of the process, the environment draws an i.i.d.\ sequence of Bernoulli random variables $X_1(v), X_2(v),\l$ with some unknown fixed parameter $q_v \in[0,1]$ for each agent $v\in V$. The active set at time $t$ is then defined as $S_t = \{v \in V \mid X_t(v) = 1 \}$. Note that, unlike the previous setting, now $\sum_{v\in V} q_v \neq 1$ in general.

We state an upper bound on the regret that the network incurs if all agents run OMD with an oblivious network interface (for a proof, see Appendix~\ref{s:many}). Our upper bound is expressed in terms of a constant depending on the probabilities of activating each agent and such that $Q \le 1.6(\alpha_G+1)$. The result holds for any differentiable function $g\colon \cx\to\R$, $\sigma$-strongly convex with respect to some norm $\norm{\cdot}$. 
\begin{theorem}
\label{t:upper-coop-many}
Consider a network $G = (V,E)$ of $N$ agents. Assume that, at each time step $t$ each agent $v$ is independently activated with probability $q_v \in [0,1]$. If all agents run OMD with an oblivious network interface and using $g_t = \frac{\sqrt{t}}{\eta}g$, for $\eta > 0$, the network regret satisfies
\[
	\E[R_T]
\le
	\left( \frac{D}{\eta}
	+\frac{\eta L^2}{2 \sigma} \right) \sqrt{Q T}
\]
for some nonnegative $Q \le 1.6(\alpha_G+1)$, $D \ge \sup g$, and $L \ge \sup \norm{\nabla\ell{}_{t}}_{*}$, where $\norm{\cdot}_{*}$ is the dual norm of $\norm{\cdot}$. In particular, choosing $\eta = \sqrt{2 \sigma D}/L$ gives
$
	\E[R_T] \le L\sqrt{2D Q T / \sigma}
$.
\end{theorem}
In order to compare the previous upper bound to Theorem~\ref{t:upper-coop}, consider the case $q_v = q$ for all $v\in V$. Without loss of generality, assume $q>0$ (the regret is zero when $q$ vanishes). Then
\[
	Q
=
	Q(q)
=
	\frac{1}{N} 
	\sum_{v\in V} \frac{1 - (1-q)^N}{1 - (1 - q)^{|\cn_v|}}
\]
(for a proof, see Theorem~\ref{t:upper-coop-many-appe} in Appendix~\ref{s:many} and proceed as in the proof of Lemma~\ref{lm:comb-tech}). A direct computation of the $\mathrm{sign}$ of the first derivative of the addends $q\mapsto \frac{1 - (1-q)^N}{1 - (1 - q)^{|\cn_v|}}$ shows that these functions are decreasing in $q$, hence
$
	1
=
	\lim_{q \to 1^-}
	Q(q)
\le
	Q 
\le 
	\lim_{q\to 0^+} 
	Q(q)
= 
	\sum_{v\in V} 
		\frac{1}{|\cn_v|}
\le
	\alpha_G
$
where the last inequality follows by Lemma~\ref{lm:bound-alpha}. Note that the lower bound $Q \ge 1$ is attained if the probabilities of picking agents at each time step are all $1$. In this case all agents are activated at each time step, the graph structure over the set of agents becomes irrelevant and the model reduces to a single-agent problem. The inequality $Q(q) \le \alpha_G$ is not a coincidence due to the constant $q$. Indeed, one can prove that this is always the case, up to a small constant factor (for a proof., see Lemma~\ref{l:bandCoop} in Appendix~\ref{s:many}).

The previous results shows that paying the average price of multiple activations is never worse (up to a small constant factor) than drawing a single agent per time step, and it can be significantly better.
A similar argument shows a tighter bound $Q\le \max\{3, \alpha_G\}$ when the activation probabilities satisfy $\sum_{v\in V} q_v = 1$, which allows to recover the upper bound on the network regret proven in Theorem~\ref{t:upper-coop}. This is consistent with the intuition that ---in expectation--- picking a single agent at random according to a distribution $q = (q_1,\dots,q_N)$ is the same as picking each $v$ independently with probability $q_v$. Similarly to the case $|S_t|=1$, the previous result gives a tighter upper bound on the network regret in terms of the independence number $\alpha' \le \alpha$ of the subgraph induced by the subset $V'$ of $V$ containing all agents $v$ with $q_v >0$.
Note that the setting discussed in this section smoothly interpolates between the single-agent setting ($q_v = 1$ for all $v$), cooperative learning with one agent stochastically activated at each time step ($\sum_v q_v = 1$), and beyond ($\sum_v q_v < 1$), where a non trivial fraction of the total number rounds is skipped.

%% file: lower.tex

\section{Lower Bound for Stochastic Activations}
\label{s:lower}
In this section we show that, for any communication network $G$ with stochastic agent activations, the best possible regret rate is of order $\Omega\big(\sqrt{\alpha_G T}\big)$. This holds even when agents are not restricted to use an oblivious network interface. The idea is that if the distribution from which active agents are drawn is supported on an independent set of cardinality $\alpha_G$, then the problem reduces to that of an edgeless graph with $\alpha_G$ agents. 
\begin{theorem}
\label{t:lower}
There exists a convex decision set in $\R^d$ such that, for each communication network $G$ and for arbitrary (and possibly different) online learning algorithms run by the agents, $\E[R_T] = \Omega\big(\sqrt{\alpha T}\big)$ for some sequence $(S_1,\ell_1),\dots,(S_T,\ell_T)$, where $S_t = \{v_t\}$, $v_t$ is drawn i.i.d.\ from some fixed distribution on $V$, and the expectation is taken with respect to the random draw of the $v_1,\l,v_T$.
\end{theorem}
\begin{proof}
We sketch the proof for the case $|S_t|=1$. Let $\cx$ be the probability simplex in $\R^d$. Let $G = (V,E)$ be any communication graph and $\alpha$ its independence number. We consider linear losses defined on $\cx$. Let $q$ be the uniform distribution over a maximal independent set $A=\{a_1,\l,a_\alpha\} \s V$. Fix now any cooperative online linear optimization algorithm for this setting. Since each active agent $v_t$ belongs to $A$ for all $t\in\{1,\l,T\}$ with probability $1$, it suffices to analyze the updates of the algorithm for these agents. Indeed, no other agent incurs any loss at any time-step. Since $A$ is an independent set, each agent $a_i$ makes an update at round $t$ if and only if $v_t = a_i$. This happens with probability $q(a_i)={1}/{\alpha}$, independently of $t$. Each agent $a_i$ is therefore running an independent single-agent online linear optimization problem for an average of ${T}/{\alpha}$ rounds. It is well-known \citep[Theorem~3.2]{hazan2016introduction} that any algorithm for online linear optimization on the simplex with losses bounded in $[0,1]$ incurs $\Omega\big(\sqrt{T/\alpha}\big)$ regret over $T/\alpha$ rounds in the worst case. Consequently, the regret of the network satisfies
$
	R_T
=
	\Omega\big(
		\alpha\sqrt{ T/\alpha}
	\big)
=
	\Omega\big(
		\sqrt{\alpha T }
	\big).
$
\end{proof}
An analogous lower bound can be proven for the case of multiple agent activations per time step. Indeed, define $q_v = 1/\alpha$ for each agent $v$ belonging to some fixed maximal independent set and $q_v=0$ otherwise. This again leads to $\alpha$ independent single-agent online linear optimization problems for an average of ${T}/{\alpha}$ rounds each, and an argument similar to the one in the proof of Theorem~\ref{t:lower} gives the result.

%% file: adv.tex

\section{Adversarial Activations}
\label{s:adv}
In this section we drop the stochasticity assumption on the agents' activations and focus on the case where active agents are picked from $V$ by an adversary. The goal is to control the regret \eqref{eq:regr} for \textsl{any} individual sequence of pairs $(\ell_1,S_1),(\ell_2,S_2),\l$ where $\ell_t$ is a convex loss and $S_t \subseteq V$, without any stochastic assumptions on the mechanism generating these pairs. For the rest of this section, we focus on the special case where $|S_t|=1$ for all $t$ and denote by $v_t$ the active node at time $t$.

We start by proving that learning with adversarial activations is impossible if we use an oblivious network interface. We prove this result in the setting of prediction with expert advice with two actions and binary losses, a special case of online convex optimization. The idea of the lower bound is that if the communication network is a star graph, the environment is able to make both actions look equally good to all peripheral agents, even if one of the two actions is actually better than the other. This is done by drawing the good action at random, then activating an agent depending on the outcome of the draw. For a small fraction of the times the good action has loss one, the central agent is activated. Since the central agent shares feedback with all peripheral agents, we can amplify this loss by a factor of $N$, and thus make the good action look to all peripheral agents as bad as the bad action.
\begin{theorem}
\label{th:lower-adv}
For each $N > 3$ there exists a convex decision set in $\R^2$ and a graph $G$ with $N$ vertices such that, whenever $N$ agents are run on $G$ using instances of any online learning algorithm with an oblivious network interface, then $R_T = \Omega(T)$ for some sequence $(\ell_1,v_1),\dots,(\ell_T,v_T)$ of convex losses and active agents.
\end{theorem}
\begin{proof}
Fix $N > 3$ and let $\cx$ be the probability simplex in $\R^2$. Let $G = (V,E)$ be the star graph with central agent $a_0$, and peripheral agents $a_1, \l, a_{N-1}$. Because our losses are linear on $\cx$, the online convex optimization problem is equivalent to prediction with expert advice with two experts (or actions), and we may denote losses using loss vectors $\ell_t = \big(\ell_t(1),\ell_t(2)\big)$ where $1$ and $2$ index the actions. A \textsl{good action} $J \in \{1,2\}$ is drawn uniformly at random. Denote the other one (i.e., the \textsl{bad} one) by $\bad$. To keep notation tidy, we define loss vectors by $\ell_t = \big(\ell_t(J),\ell_t(\bad)\big)$. Fix any $\e \in \big(0, \frac{N-1}{2(N-2)}\big)$. The loss vectors $\ell_t$ are drawn i.i.d.\ at random, according to the following joint distribution:
\begin{align*}
	\P\big( \ell_t = (0,1) \big) 
	= \frac{1}{2}
\qquad
	\P\big( \ell_t = (1,0) \big)
	= \frac{1}{2} - \e + \frac{\e}{N-1}
\qquad
	\P\big( \ell_t = (0,0) \big)
	= \e - \frac{\e}{N-1}
\end{align*}
Recall that only a single agent $v_t$ is active at any time. At each time step $t$, the adversary decides whether to activate the central agent $a_0$ or a peripheral agent, depending on the realization of $\ell_t$. If $\ell_t(J)=0$, then a random peripheral agent is activated. Otherwise, we set
\begin{align*}
	\P\big( \ell_t = (1,0), \, v_t = a_0 \big) = \frac{\e}{N-1}
\quad\text{and}\quad
	\P\big( \ell_t = (1,0), \, v_t = a_i \big) = \frac{\nicefrac{1}{2} - \e}{N-1} \quad \text{$a_1,\dots,a_{N-1}$}
\end{align*}
Note that when $v_t = a_0$, then all peripheral agents receive feedback $\ell_t$. Similarly, when a peripheral agent is active at time $t$, then $a_0$ receives feedback $\ell_t$. For $b_1,b_2\in\bool$, let $E(a_i,b_1,b_2)$ be the event: agent $a_i$ receives the loss vector $\ell_t = (b_1,b_2)$ as feedback. The following statements then hold for each peripheral agent $a_i$,
\begin{align*}
	\Pr\big(E(a_i,0,1) \big) &= \frac{\nicefrac{1}{2}}{N-1}
\qquad
	\Pr\big(E(a_i,0,0) \big) = \frac{\e}{N-1} - \frac{\e}{(N-1)^2}
\\
	\Pr\big(E(a_i,1,0) \big) &= \frac{\nicefrac{1}{2} - \e}{N-1} + \frac{\e}{N-1} = \frac{\nicefrac{1}{2}}{N-1}
\end{align*}
Hence, each instance managed by a peripheral agent observes loss vectors $(1,0)$ and $(0,1)$ with the same probability proportional to $\nicefrac{1}{2}$, and loss vector $(0,0)$ with probability proportional to $\e(N-1)/(N-2)$. Since the network interface is oblivious, the instance cannot distinguish between paid and free feedback (which would reveal the good action), and incurs an expected loss of $\nicefrac{1}{2}$ each time $\ell_t \in \big\{(0,1), (1,0)\big\}$.
Using the fact that a peripheral agent is active when $\ell_t \in \big\{(0,1), (1,0)\big\}$ with probability $\nicefrac{1}{2} + \nicefrac{1}{2} - \e = 1 - \e$, the system's expected total loss is at least $\frac{1-\e}{2}T$ (we lower bound the loss of the central agent by zero). Since the expected loss of $J$ is $\big( \nicefrac{1}{2} - \e + \frac{\e}{N-1} \big) T$, the expected regret of the system satisfies
\[
	\E[R_T]
\ge
	\left( \frac{1-\e}{2} - \frac{1}{2} + \e - \frac{\e}{N-1} \right) T
\ge
	\frac{T}{8}
\]
where we picked $\e = (N-1)/(N-2)$ and used $(N-3)/(N-2)\ge \nicefrac{1}{2}$ in the last inequality. Therefore, there exists some sequence $(\ell_1,S_1),\dots,(\ell_T,S_T)$ such that $R_T \ge T/8$, concluding the proof.
\end{proof}
We complement the above negative result by showing that when algorithms are run without the oblivious network interface, and agents are free to use feedback only from a subset of their neighbors chosen with knowledge of the graph structure, then the network regret of OMD is $\scO(\sqrt{\chibar_G T})$. The quantity $\chibar_G$ is the clique-covering number of the communication graph $G$, which corresponds to the smallest cardinality of a clique cover of $G$ (a clique cover is a partition of the vertices such that the nodes in every element of the partition form a clique in the graph). The intuition behind this result is simple: fix a clique cover and let the agents in the same clique of the cover know each other. Now, if each agent ignores all feedback coming from agents in other cliques, then the agents in the same clique make exactly the same sequence of prediction and updates. Therefore, the effective number of OMD instances that are being run is equal to $\chibar_G$.

The following result holds for any differentiable function $g\colon \cx\to\R$, $\sigma$-strongly convex with respect to some norm $\norm{\cdot}$. 
\begin{theorem}
\label{t:upper-adv}
Consider a network $G = (V,E)$ of $N$ agents, a clique cover $\{K_1, \dots, K_M\}$ where $M = \chibar_G$, and let $K(v)$ be the unique element of the cover which each $v\in V$ belongs to. For any sequence $v_1,v_2,\ldots\in V$ of active agents, assume each agent $v \in V$ runs OMD using $g_t = \frac{\sqrt{t}}{\eta}g$ (with $\eta > 0$) while making updates only at those time steps $t$ such that $v_t \in K(v)$. Then the network regret satisfies
\[
	\E[R_T]
\le
	\left( \frac{D}{\eta}
	+\frac{\eta L^2}{2 \sigma} \right) \sqrt{\chibar_G T}
\]
where $D \ge \sup g$, $L \ge \sup \norm{\nabla\ell{}_{t}}_{*}$, and $\norm{\cdot}_{*}$ is the dual norm of $\norm{\cdot}$. In particular, choosing $\eta = \sqrt{2 \sigma D}/L$ gives
$
	\E[R_T] \le L\sqrt{2D \chibar_G T / \sigma}
$.
\end{theorem}
\begin{proof}
Fix any clique $K_c$ and any $v \in K_c$. Let $T_c$ be the time steps such that $v_t \in K_c$. Since each agent $v \in K_c$ ignores the feedback coming from other cliques, the nodes in $K_c$ perform exactly the same updates, and therefore make exactly the same predictions. This means that, for any $t \in T_c$, the predictions in the set $\big\{\bx_t(v) \mid v \in K_c\big\}$ are all equal to the same common value denoted by $\bx_t(K_c)$. Fix any $\bx \in \cx$ and, for any $t\in T_c$, let
$
	r_t(K_c)
=
	\ell_t\big(\bx_t(K_c)\big) - \ell_t (\bx)
$.
By Theorem~\ref{t:omd} we have that
\[
	\sum_{t \in T_c} r_t(K_c)
\le
	\left(\frac{D}{\eta} + \frac{\eta L^2}{2\sigma}\right)\sqrt{T_c}~.
\]
Therefore, recalling that $r_t(v_t) = \ell_t\big(\bx_t(v_t)\big) - \ell_t (\bx)$ and using Jensen's inequality,
\[
	\sum_{t=1}^T r_t(v_t)
=
	\sum_{c=1}^{\chibar_G} \sum_{t \in T_c} r_t(K_c)
\le
	\sum_{c=1}^{\chibar_G} \left(\frac{D}{\eta} + \frac{\eta L^2}{2\sigma}\right)\sqrt{T_c}
\le
	\left(\frac{D}{\eta} + \frac{\eta L^2}{2\sigma}\right)\sqrt{\chibar_G T}
\]
concluding the proof.
\end{proof}
Theorems~\ref{th:lower-adv} and~\ref{t:upper-adv} show that with adversarial activations the knowledge of the graph is crucial for learning (e.g., for achieving sublinear regret). Since $\chibar_G \ge \alpha_G$, it is not clear whether the better rate $\sqrt{\alpha_G T}$ can be proven in the adversarial activation setting when agents do not use the oblivious network interface.

%% file: conclusions.tex

\section{Conclusions}
\label{s:conclusions}
We introduced a cooperative learning setting in which agents, sitting on the nodes of a communication network, run instances of an online learning algorithm with the common goal of minimizing their regret. In order to investigate how the knowledge of the graph topology affects regret in cooperative online learning under different activation mechanisms, we introduced the notion of oblivious network interface. This prevents agents from doing any network-specific tuning or even accessing their neighborhood structure. When activations are stochastic, we showed that sharing losses among neighbors is enough to guarantee optimal regret rates even with the oblivious network interface. Surprisingly, when activations are adversarial the situation changes completely. There exist problem instances in which any algorithm that runs with the oblivious network interface suffers linear regret. In this case knowing graph structure is not only necessary to perform optimally, but even to have sublinear regret.

Other interesting variants of this settings could be studied in the future. For example, at the beginning of each round, active agents could be allowed to ask the predictions of some of their neighbors, and base their prediction upon it. In this case, we conjecture that the optimal regret rate would scale with the dominating number $\delta_G$ of the graph, which is always smaller or equal to the independence number.

%% file: manyInApp.tex

\section{Stochastic Activations: Multiple Agents}
\label{s:many}
In this section we present all missing results related to the stochastic activation model with multiple activations per time step. Recall that, at the beginning of the process, the environment draws an i.i.d.\ sequence of Bernoulli random variables $X_1(v), X_2(v),\l$ with some unknown fixed parameter $q_v \in[0,1]$ for each agent $v\in V$. The active set at time $t$ is then defined as $S_t = \{v \in V \mid X_t(v) = 1 \}$. Note that, unlike when only one agent is active at each time step, now $\sum_{v\in V} q_v \neq 1$ in general. Before the main result, we give some definitions and prove a technical combinatorial lemma that is leveraged in the analysis.

Denote by $V'$ the set of all agents $v\in V$ such that $q_v>0$. For each $v \in V'$, let
\begin{equation}
\label{e:c_v}
	c_v 
=
	\sum_{S\s \{1,\l,N\}\m\{v\}} \frac{\lambda_{S,v}}{1+|S|}
\end{equation}
where the convex coefficients $\lambda_{S,v}$ are defined by
\[
	\left(
		\prod_{w=1}^N q_w
	\right)
	\left(
		\prod_{u\in \{1,\l,N\} \m ( \{v\} \cup S )} (1-q_u)
	\right)
\]
Let also $Q_v$ be the probability
\begin{equation}
\label{e:Q_v-many}
	\Pr\left( v \in \bigcup_{w\in S_t} \cn_{w} \right) 
= 
	1 - \prod_{w \in \cn_v} \big( 1 - q_w \big)
>
	0
\end{equation}
that agent $v$ is updated at time $t$ ---note that $Q_v$ is independent of $t$. 
\begin{lemma}
\label{lm:comb-tech}
Let $X(1), \l, X(m)$ be independent Bernoulli random variables with strictly positive parameters $q_1, \l, q_m$ respectively. Then, for all $v\in\{1,\l,m\}$,
\[
	\E\left[
		\frac{X(v)}
		{\sum_{w=1}^m X(w)}
	\right]
=
	q_v c_v
\]
where we define $X(v)/\sum_{w=1}^m X(w) = 0$ when $X(v) = 0$.
\end{lemma}
\begin{proof}
Fix any $v\in \{1,\l,m\}$. Let $S_v$ be the set $\{1,\l,m\}\m\{v\}$ and let  $\cf_v$ be the $\sigma$-algebra generated by 
$
	\big\{X(w) \mid w \in S_v\big\}
$.  
Then
\begin{align*}
	\E \left[
		\frac{X(v)}{\sum_{w=1}^m X(w)}
	\right]
=
	\E \left[
		\E \left[
			\frac{X(v)}{\sum_{w=1}^m X(w)}
			\bigg| \cf_v
		\right]
	\right]
=
	q_v \,
	\E \left[
		\frac{1}{1+ \sum_{w\in S_v} X(w)}
	\right]
\end{align*}
Denote the last expectation by $c_v$. Since for all $x\neq 0$, 
$
	\int_0^{\iop} e^{-t x} \dif t 
= 
	\frac{1}{x}
$,
Fubini's theorem yields
\begin{align*}
	c_v
& = 
	\int_0^{\iop}	
	\E \left[
		e^{-t \big(1 + \sum_{w \in S_v} X(w)\big)}
	\right]
	\dif t
\\ & =
	\int_0^{\iop}	
		e^{-t} \prod_{w \in S_v} \E\left[e^{-t X_t(w)}\right]
	\dif t
\\ & =
	\int_0^{\iop}	
		e^{-t}\prod_{w \in S_v} \big(q_w e^{-t}+1-q_w\big)
	\dif t
\\ & =	
	\int_0^1	
		\prod_{w \in S_v} (q_w x+1-q_w)
	\dif x
\\ & =
	\int_{0}^{1}
	\sum_{S \s S_v }
		x^{\left|S\right|}
		\left(
			\prod_{w\in S} q_w
		\right)
		\left(
			\prod_{u\in S_v \m S} (1-q_u)
		\right)
	\dif x
\end{align*}
Now set
$
	\lambda_{S,v}
=
	\big( \prod_{w\in S} q_w \big)
	\big( \prod_{S_v \m S} (1-q_u) \big)
$ 
and note that $\sum_{S\s S_v} \lambda_{S,v} = \prod_{w\in S_v} ( q_w + 1 - q_w ) = 1$. Substituting $\lambda_{S,v}$ in the last identity gives
\[
	c_v
= 
	\sum_{S \s S_v }
	\lambda_{S,v}
	\int_{0}^{1}
		x^{\left|S\right|}
	\dif x
=
	\sum_{S \s S_v }
	\frac{\lambda_{S,v}}{1+|S|}
\]
\end{proof}
We now give an upper bound on the regret that the network incurs if all agents run OMD with an oblivious network interface. Our upper bound is expressed in terms of a constant depending on the probabilities of activating each agent and such that $Q \le 1.6(\alpha_G+1)$. The result holds for any differentiable function $g\colon \cx\to\R$, $\sigma$-strongly convex with respect to some norm $\norm{\cdot}$. 
\begin{theorem}
\label{t:upper-coop-many-appe}
Consider a network $G = (V,E)$ of $N$ agents. Assume that, at each time step $t$ each agent $v$ is independently activated with probability $q_v \in [0,1]$. If all agents run OMD with an oblivious network interface and using $g_t = \frac{\sqrt{t}}{\eta}g$, for $\eta > 0$, the network regret satisfies
\[
	\E[R_T]
\le
	\left( \frac{D}{\eta}
	+\frac{\eta L^2}{2 \sigma} \right) \sqrt{Q T}
\]
where
$
	Q
=
	\sum_{v\in V'}
	(q_v c_v)/Q_v 
$,
$D \ge \sup g$, $L \ge \sup \norm{\nabla\ell{}_{t}}_{*}$, and $\norm{\cdot}_{*}$ is the dual norm of $\norm{\cdot}$. In particular, choosing $\eta = \sqrt{2 \sigma D}/L$ gives
$
	\E[R_T] \le L\sqrt{2D Q T / \sigma}
$.
\end{theorem}
\begin{proof}
Fixing an arbitrary $\bx\in \cx$, setting 
$
	r_t(v)
=
	\ell_t\big(\bx_t(v)\big)
	- \ell_t (\bx)
$, 
and proceeding as in Theorem~2 yields, for each $v\in V'$,
\begin{equation}
\label{e:fdm-set}
	\E \left[ \sum_{t=1}^{T} 
		r_t(v)	
	\right]
\le
	\left(\frac{D}{\eta} + \frac{\eta L^2}{2\sigma} \right)
	\sqrt{\frac{T}{Q_v}}
\end{equation}
Now we write ${\dt \E[R_T] = \sup_{\bx\in\cx} \E\big[R_T(\bx)\big] }$, where
\begin{align}
	\E\big[R_T (\bx)\big]
& =
	\E \left[
		\sum_{t=1}^{T}
		\frac{1}{\sum_{w\in V} X_t(w)}
		\sum_{v\in V'}
		r_{t}(v) 
		X_t(v) 
	\right]
\nonumber \\ & =
	\sum_{t=1}^{T}
	\sum_{v\in V'}
	\E \left[
		\frac{X_t(v)}{\sum_{w\in V} X_t(w)}
	\right]
	\E \big[
		r_{t} (v) 
	\big]
\nonumber \\ & =
	\label{e:regr-part}
	\sum_{v\in V'}
	q_v c_v
	\sum_{t=1}^{T}
	\E \big[
		r_{t} (v) 
	\big]
\end{align}
and the last identity follows by Lemma~\ref{lm:comb-tech}. Putting identity \eqref{e:regr-part} and inequality \eqref{e:fdm-set} together gives
\[
	\E\big[R_T (\bx)\big]
\le
	\left( \sum_{v\in V'}
		q_v c_v
	\sqrt{\frac{1}{Q_v}} \right)
	\left(\frac{D}{\eta} + \frac{\eta L^2}{2\sigma} \right)
	\sqrt{T}
\le 
	\sqrt{\sum_{v\in V'} \frac{q_v c_v}{Q_v}}
	\left(\frac{D}{\eta} + \frac{\eta L^2}{2\sigma} \right)
	\sqrt{T}
\]
where in the last inequality we used Jensen inequality and $\sum_{v\in V'} q_v c_v \le 1$. This concludes the proof.
\end{proof}
We now prove that the inequality $Q(q) \le \alpha_G$ is always true up to a small constant factor.
\begin{lemma}
\label{l:bandCoop}
Let $G=(V,E)$ be an undirected graph. For all $v\in V$, choose numbers $q_v \in (0,1]$ and define $c_v$ and $Q_v$ as in \eqref{e:c_v} and \eqref{e:Q_v-many} respectively. Then 
\[
	Q = \sum_{v\in V} \frac{q_v c_v}{Q_v}
\le
	\frac{\alpha_G + 1}{1-e^{-1}}
\]
\end{lemma}
\begin{proof}
Let $P_v = \sum_{w\in \cn_v} q_w$, $V_1 = \big\{ v\in V \mid P_v \ge 1\big\}$, and $V_0 = \big\{ v\in V \mid P_v < 1\big\}$. We begin by splitting the sum as follows
\[
	\sum_{v\in V} \frac{q_v c_v}{Q_v}
=
	\sum_{v\in V_1} \frac{q_v c_v}{Q_v}
	+\sum_{v\in V_0} \frac{q_v c_v}{Q_v}
\]
We upper bound the two terms separately. Since the minimum $\min_{v\in V_1} Q_v$ is attained when $q_v = 1/|\cn_v|$ for all $v\in \cn_v$, we can lower bound, for each $v\in V_1$,
\[
	Q_v \ge 1 - \left(1 - \frac{1}{|\cn_v|}\right)^{|\cn_v|} \ge 1 - e^{-1}
\]
This together with $\sum_{v\in V} q_v c_v \le 1$ yields
\[
	\sum_{v\in V_1} \frac{q_v c_v}{Q_v} \le \frac{1}{1 - e^{-1}}
\]
To upper bound the sum over $V_0$, we first use the inequality $1-x \le e^{-x}$ that holds for all $x\in[0,1]$. Setting $x = q_w$ gives
\[
	Q_v \ge 1 - \exp \left( - \sum_{w \in \cn_v} q_w \right) = 1 - e^{- P_v}
\]
For all $v\in V_0$, we can then use the inequality $1-e^{-x} \ge (1-e^{-1})x$, holding for all $x\in[0,1]$. Setting $x = P_v < 1$ we conclude that $Q_v \ge (1-e^{-1})P_v$ for all $v\in V_0$. Finally, using $c_v \le 1$ we can write
\[
	\sum_{v\in V_0} \frac{c_v q_v}{Q_v} \le \frac{1}{1-e^{-1}} \sum_{v\in V} \frac{q_v}{P_v} \le 	\frac{\alpha_G}{1-e^{-1}}
\]
where the last inequality follows by Lemma~\ref{lm:bound-alpha}. Putting everything together gives the result.
\end{proof}
\textcolor{white}{.}

%% file: master.bbl
\begin{thebibliography}{18}
\providecommand{\natexlab}[1]{#1}
\providecommand{\url}[1]{\texttt{#1}}
\expandafter\ifx\csname urlstyle\endcsname\relax
  \providecommand{\doi}[1]{doi: #1}\else
  \providecommand{\doi}{doi: \begingroup \urlstyle{rm}\Url}\fi

\bibitem[Alon et~al.(2015)Alon, Cesa-Bianchi, Dekel, and Koren]{alon2015online}
Noga Alon, Nicolo Cesa-Bianchi, Ofer Dekel, and Tomer Koren.
\newblock Online learning with feedback graphs: Beyond bandits.
\newblock In \emph{Annual Conference on Learning Theory}, volume~40. Microtome
  Publishing, 2015.

\bibitem[Awerbuch and Kleinberg(2008)]{awerbuch2008competitive}
Baruch Awerbuch and Robert Kleinberg.
\newblock Competitive collaborative learning.
\newblock \emph{Journal of Computer and System Sciences}, 74\penalty0
  (8):\penalty0 1271--1288, 2008.

\bibitem[Cesa-Bianchi et~al.(2016)Cesa-Bianchi, Gentile, Mansour, and
  Minora]{cesa2016delay}
Nicolo Cesa-Bianchi, Claudio Gentile, Yishay Mansour, and Alberto Minora.
\newblock Delay and cooperation in nonstochastic bandits.
\newblock \emph{JMLR Workshop and Conference Proceedings (COLT 2016)},
  49:\penalty0 605--622, 2016.

\bibitem[Duchi et~al.(2012)Duchi, Agarwal, and Wainwright]{duchi2012dual}
John~C Duchi, Alekh Agarwal, and Martin~J Wainwright.
\newblock Dual averaging for distributed optimization: Convergence analysis and
  network scaling.
\newblock \emph{IEEE Transactions on Automatic control}, 57\penalty0
  (3):\penalty0 592--606, 2012.

\bibitem[Freund et~al.(1997)Freund, Schapire, Singer, and
  Warmuth]{freund1997using}
Yoav Freund, Robert~E Schapire, Yoram Singer, and Manfred~K Warmuth.
\newblock Using and combining predictors that specialize.
\newblock In \emph{In Proceedings of the Twenty-Ninth Annual ACM Symposium on
  the Theory of Computing}. Citeseer, 1997.

\bibitem[Griggs(1983)]{griggs1983lower}
Jerrold~R Griggs.
\newblock Lower bounds on the independence number in terms of the degrees.
\newblock \emph{Journal of Combinatorial Theory, Series B}, 34\penalty0
  (1):\penalty0 22--39, 1983.

\bibitem[Hazan(2016)]{hazan2016introduction}
Elad Hazan.
\newblock Introduction to online convex optimization.
\newblock \emph{Foundations and Trends{\textregistered} in Optimization},
  2\penalty0 (3-4):\penalty0 157--325, 2016.

\bibitem[Hosseini et~al.(2013)Hosseini, Chapman, and
  Mesbahi]{hosseini2013online}
Saghar Hosseini, Airlie Chapman, and Mehran Mesbahi.
\newblock Online distributed optimization via dual averaging.
\newblock In \emph{52nd Annual IEEE Conference on Decision and Control (CDC)},
  pages 1484--1489. IEEE, 2013.

\bibitem[Mannor and Shamir(2011)]{mannor2011bandits}
Shie Mannor and Ohad Shamir.
\newblock From bandits to experts: On the value of side-observations.
\newblock In \emph{Advances in Neural Information Processing Systems}, pages
  684--692, 2011.

\bibitem[Mart{\'\i}nez-Rubio et~al.(2018)Mart{\'\i}nez-Rubio, Kanade, and
  Rebeschini]{martinez2018decentralized}
David Mart{\'\i}nez-Rubio, Varun Kanade, and Patrick Rebeschini.
\newblock Decentralized cooperative stochastic multi-armed bandits.
\newblock \emph{arXiv preprint arXiv:1810.04468}, 2018.

\bibitem[McQuade and Monteleoni(2012)]{mm12}
Scott McQuade and Claire Monteleoni.
\newblock Global climate model tracking using geospatial neighborhoods.
\newblock In \emph{Proc. Twenty-Sixth AAAI Conference on Artificial
  Intelligence, Special Track on Computational Sustainability and AI}, pages
  335--341, 2012.

\bibitem[McQuade and Monteleoni(2017)]{mm17}
Scott McQuade and Claire Monteleoni.
\newblock Spatiotemporal global climate model tracking.
\newblock \emph{Large-Scale Machine Learning in the Earth Sciences; Data Mining
  and Knowledge Discovery Series. Srivastava, A., Nemani R., Steinhaeuser, K.
  (Eds.), CRC Press, Taylor \& Francis Group}, 2017.

\bibitem[Orabona et~al.(2015)Orabona, Crammer, and
  Cesa-Bianchi]{orabona2015generalized}
Francesco Orabona, Koby Crammer, and Nicolo Cesa-Bianchi.
\newblock A generalized online mirror descent with applications to
  classification and regression.
\newblock \emph{Machine Learning}, 99\penalty0 (3):\penalty0 411--435, 2015.

\bibitem[Sahu and Kar(2017)]{sahu2017dist}
Anit~Kumar Sahu and Soummya Kar.
\newblock Dist-{H}edge: A partial information setting based distributed
  non-stochastic sequence prediction algorithm.
\newblock In \emph{IEEE Global Conference on Signal and Information Processing
  (GlobalSIP)}, pages 528--532. IEEE, 2017.

\bibitem[Scaman et~al.(2018)Scaman, Bach, Bubeck, Massouli{\'e}, and
  Lee]{scaman2018optimal}
Kevin Scaman, Francis Bach, S{\'e}bastien Bubeck, Laurent Massouli{\'e}, and
  Yin~Tat Lee.
\newblock Optimal algorithms for non-smooth distributed optimization in
  networks.
\newblock In \emph{Advances in Neural Information Processing Systems}, pages
  2745--2754, 2018.

\bibitem[Shahrampour and Jadbabaie(2018)]{shahrampour2018distributed}
Shahin Shahrampour and Ali Jadbabaie.
\newblock Distributed online optimization in dynamic environments using mirror
  descent.
\newblock \emph{IEEE Transactions on Automatic Control}, 63\penalty0
  (3):\penalty0 714--725, 2018.

\bibitem[Shalev-Shwartz(2012)]{shalev2012online}
Shai Shalev-Shwartz.
\newblock Introduction to online convex optimization.
\newblock \emph{Foundations and Trends{\textregistered} in Machine Learning},
  4\penalty0 (2):\penalty0 107--194, 2012.

\bibitem[Zhao et~al.(2019)Zhao, Yu, Zhao, and Liu]{zhao2019decentralized}
Yawei Zhao, Chen Yu, Peilin Zhao, and Ji~Liu.
\newblock Decentralized online learning: Take benefits from others' data
  without sharing your own to track global trend.
\newblock \emph{arXiv preprint arXiv:1901.10593}, 2019.

\end{thebibliography}
